\documentclass{article} 
\usepackage{iclr2025_conference,times}

\usepackage{amsmath,amsfonts,bm}

\def\eqref#1{equation~\ref{#1}}

\def\1{\bm{1}}

\DeclareMathAlphabet{\mathsfit}{\encodingdefault}{\sfdefault}{m}{sl}
\SetMathAlphabet{\mathsfit}{bold}{\encodingdefault}{\sfdefault}{bx}{n}

\usepackage{hyperref}
\usepackage{url}
\usepackage{algorithm}
\usepackage{algorithmic}
\usepackage{amsmath}
\usepackage{amsfonts}
\usepackage{amssymb}
\usepackage{amsthm}
\usepackage{multirow}
\usepackage{multicol}
\usepackage{subfiles}
\newtheorem{theorem}{Theorem}
\newtheorem{lemma}{Lemma}
\usepackage{tcolorbox}
\usepackage{xcolor}
\usepackage{float}
\usepackage{booktabs}

\title{Beyond Success and Failure: Length-Aware Contrastive Learning for GUI Agents}

\author{Chengyang Gu\dag \& Hui Xiong\thanks{Correspondence to: \texttt{xionghui@hkust-gz.edu.cn}
} \\
Artificial Intelligence Thrust, Information Hub\\
The Hong Kong University of Science and Technology (Guangzhou),  Guangzhou, China \\
\dag \texttt{cgu893@connect.hkust-gz.edu.cn}
\And
Le Zhang, Jingbo Zhou, Yu Shi, Siqi Bao, Zheng-Fan Wu \& Hua Wu \\
Baidu Inc. Beijing, China\\
\\
\And
Yize Chen \\
Department of Electrical and Computer Engineering\\
University of Alberta, Edmonton, Alberta, Canada\\
}

\iclrfinalcopy 
\begin{document}

\maketitle

\begin{abstract}
Graphical User Interface (GUI) agents powered by Multimodal Large Language Models (MLLMs) have shown strong potential for automating tasks across diverse digital environments, where reinforcement learning (RL) has become a dominant training paradigm. However, widely used methods such as Group Relative Policy Optimization (GRPO) suffer from reward–gradient misalignment, leading to inefficient and unstable optimization. Recent work addresses this issue by reformulating RL with verifiable rewards (RLVR) as contrastive or classification-based objectives, which improve stability by eliminating problematic gradient behaviors. Despite this progress, existing contrastive RLVR methods rely primarily on outcome-level supervision and fail to capture fine-grained differences in trajectory quality within the same outcome category. In this paper, we propose Length-Aware Contrastive Learning for GUI Agents (LACL-GUI), a contrastive RLVR framework that incorporates trajectory-level quality signals into policy optimization. LACL-GUI introduces structured preferences within both successful and failed trajectories, encouraging concise successful executions and differentiating failure quality based on divergence from successful trajectories, while preserving optimization stability. Experiments on GUI agent benchmarks show that LACL-GUI provides more effective learning signals and consistently improves agent performance over prior methods, highlighting the value of trajectory-level supervision in contrastive RLVR.
\end{abstract}

\section{Introduction}

Graphical User Interface (GUI) agents powered by Multimodal Large Language Models (MLLMs) \citep{bai2025qwen3, wang2025ui, zhang2024android} have demonstrated remarkable potential in automating tasks across web, mobile, and desktop platforms. Prior work typically relies on supervised fine-tuning (SFT) \citep{xu2024aguvis, qin2025ui}, which demands substantial high-quality synthetic data and often struggles to generalize to unseen interfaces \citep{luo2025gui, lu2026ui}. Unlike SFT, reinforcement learning (RL) enables agents to learn from trial-and-error interaction with environmental feedback, continuously refining policies without exhaustive labeled demonstrations \citep{zhou2026gui, hu2026gui, yuan2026se}. As a result, RL has emerged as a dominant paradigm for training advanced GUI agents \citep{lu2025ui, hu2026gui}.

Among RL-based approaches, Group Relative Policy Optimization (GRPO) is one of the dominant methods due to its strong empirical performance and computational efficiency \citep{shao2024deepseekmath, guo2025deepseek}. By leveraging group-wise reward normalization, GRPO eliminates the need for a separate critic network while maintaining stable policy optimization. However, recent studies reveal a misalignment between reward signals and actual gradient allocation, where optimization pressure is disproportionately influenced by policy likelihood rather than trajectory quality \citep{deng2026group, zhang2026revisiting}. \citep{zhai2026rewards} theoretically identify two such gradient pathologies in GRPO: Gradient Misassignment in Positives, where high-probability tokens receive disproportionately large updates while low-probability tokens receive weak gradients despite belonging to successful trajectories; and Gradient Domination in Negatives, where gradient magnitudes concentrate excessively on a small number of high-probability negative tokens, allowing them to dominate optimization. Consequently, verifiable rewards are not faithfully reflected in policy updates, limiting both optimization efficiency and training stability.

To address reward-gradient mismatch in GRPO-style optimization, recent work has increasingly reformulated RLVR as contrastive or classification-based objectives rather than scalar reward-weighted optimization~\citep{oord2018representation}. Representative examples include Rewards as Labels (REAL) \citep{zhai2026rewards}, CLIPO \citep{cui2026clipo}, and Harmony \citep{yu2026harmony}, which improve optimization stability, robustness, and learning effectiveness through comparison-based objectives, with REAL in particular showing that classification-based optimization induces monotonic and bounded gradient weighting that mitigates the pathologies of GRPO. However, this line of work still reduces every trajectory to a single binary label, leaving the resulting classification-based objectives blind to any structure within positive or negative groups themselves.

Despite these optimization advantages, existing contrastive RLVR methods still derive supervision primarily from outcome-level labels. Consequently, trajectories sharing the same outcome are often treated equivalently, despite potentially large differences in execution quality, efficiency, and proximity to success \citep{yu2026outcome,zhan2026graphpographbasedpolicyoptimization,stepbranpo,zheng2026adaptive}. This coarse-grained supervision limits RLVR methods' capabilities in capturing fine-grained trajectory preferences~\citep{zhang2026revisiting}. 
This limitation is particularly pronounced for GUI agents, where successful trajectories may vary significantly in efficiency, while failed trajectories can differ substantially in how close they come to completing the task. The central challenge is therefore how to incorporate trajectory-level quality indicators into contrastive RLVR, while preserving reliable policy optimization.

To address this challenge, we propose \textit{\textbf{L}ength-\textbf{A}ware \textbf{C}ontrastive \textbf{L}earning for \textbf{GUI} Agents} \textbf{(LACL-GUI)}, a length-aware contrastive RLVR framework for GUI agent training. LACL-GUI work on RLVR-based objectives and introduces structured preferences within both successful and failed trajectory groups: for successful trajectories, we encourage shorter solutions by assigning stronger optimization pressure to concise executions while down-weighting redundant ones; for failed trajectories, we introduce divergence length, which measures how long a failed trajectory remains aligned with a reference successful execution before deviating, treating later-diverging failures as more informative negative samples and applying weaker suppression. These preferences are incorporated through outcome-preserving, zero-mean logit modulations, so that optimization is refined within each outcome category without altering the underlying success–failure decision boundary. As a result, LACL-GUI successfully provides fine-grained efficiency and failure-quality guidance on top of classification-based RLVR, while preserving its optimization stability and gradient safety. We evaluate LACL-GUI on OSWorld benchmark \citep{xie2024osworld} using Qwen3-VL-Thinking (4B and 8B)  as backbone models. Experiments show LACL-GUI consistently outperforms baselines at both scales, improving task success rate by 2.8\% over GRPO and 2.7\% over REAL on 8B backbone within 50 training iterations. 

\section{Related Works}

\subsection{GUI Agents}
GUI Agents aim to execute tasks across web, mobile, and desktop environments by mapping high-level instructions to grounded interface actions. Early approaches rely on structured representations, such as HTML \citep{deng2023mind2web} and DOM trees \citep{wen2024autodroid}. However, these methods depend heavily on the coverage and quality of such representations, which are often brittle and inconsistent across platforms \citep{wu2026gui, tang2026lpo}. Recent approaches leverage Multimodal Large Language Models (MLLMs) to enable more general and visually grounded interactions. Given natural language instructions and interface observations, MLLM-based GUI agents such as UI-TARS \citep{qin2025ui}, OS-Atlas \citep{wu2025atlas}, and UGround \citep{gou2025navigating} parse model responses into executable actions through predefined prompting and action schemas. In this work, we follow the MLLM-based paradigm and adopt the Qwen3-VL series model \citep{bai2025qwen3} with its action parser as the backbone for policy optimization.

\begin{figure*}[t]
    \centering
    \includegraphics[width=0.99\textwidth]{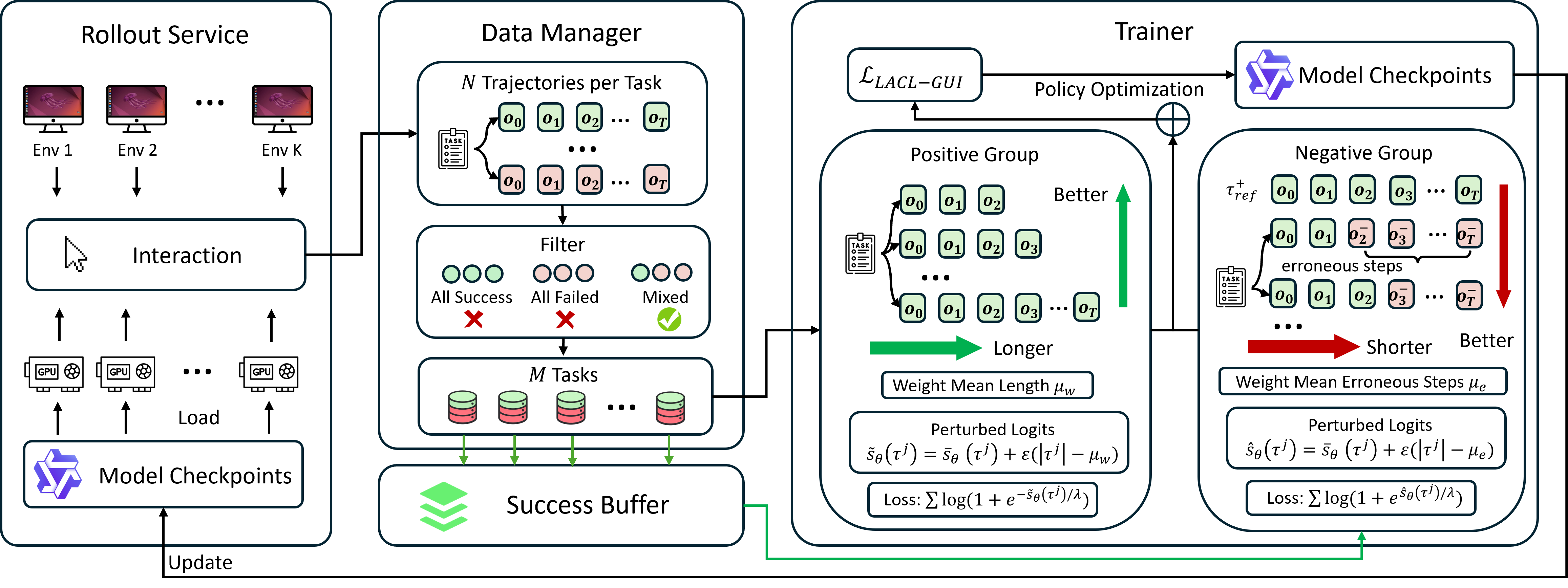}
    \caption{
    Overview of LACL-GUI framework. A decoupled architecture collects informative trajectory groups and maintains a \textit{Success Buffer} for reference trajectories. The proposed objective introduces length-aware positive preference and divergence-aware negative refinement on top of REAL contrastive objective.
    }
    \label{fig:lacl_framework}
\end{figure*}

\subsection{Reinforcement Learning}
Reinforcement learning (RL) has become a promising paradigm for enhancing the reasoning and decision-making capabilities of LLM and MLLM-based agents \citep{guo2025deepseek, jaech2024openai, team2025kimi}, particularly in multi-turn interactive environments. Early approaches largely rely on reinforcement learning from human feedback (RLHF), including methods such as Direct Preference Optimization (DPO) \citep{rafailov2023direct}, which derive reward signals from offline human preference data. However, RLHF suffers from high annotation costs and noisy reward estimation \citep{kaufmann2023survey}. Recent work therefore shifts toward Reinforcement Learning with Verifiable Rewards (RLVR), where rewards are directly obtained from task outcomes. Group Relative Policy Optimization (GRPO) \citep{shao2024deepseekmath} is a representative RLVR approach that removes the need for a critic model through group-wise advantage estimation. Subsequent variants, including Dr.GRPO \citep{liu2025understanding}, DAPO \citep{yu2026dapo}, and GSPO \citep{zheng2025group}, further improve training stability and efficiency. Nevertheless, recent studies such as REAL \citep{zhai2026rewards} reveal that GRPO-style objectives suffer from gradient misalignment, where policy updates can be dominated by token likelihood rather than trajectory quality. These findings motivate exploration of alternative RLVR objectives beyond conventional reward-weighted operations.

Recent works also explore finer-grained credit assignment for GUI agents beyond sparse outcome rewards. GiGPO \citep{feng2026group} introduces milestone-based rewards through manually designed environment-specific rules, while other approaches estimate intermediate progress using pretrained representation models \citep{zheng2026adaptive, wang2026rlanything}. However, these approaches either require task-specific supervision or additional models, limiting their scalability across diverse GUI environments. In contrast, our approach derives trajectory-level quality signals directly from successful and failed rollouts within a unified contrastive optimization framework.

\subsection{Contrastive Learning for RLVR}
Contrastive learning is a widely used representation learning paradigm that encourages positive samples to be closer while separating negative samples \citep{sammani2023visualizing}. Since RLVR methods typically rely on binary outcome rewards \citep{parthasarathi2025grpo}, recent studies show that objectives such as GRPO can be interpreted from a classification perspective \citep{mroueh2025reinforcement}. This perspective has inspired methods that replace reward-weighted optimization with discriminative or contrastive objectives. REAL \citep{zhai2026rewards} formulates RLVR as a binary classification problem using BCE loss, providing a stable optimization alternative by contrasting successful and unsuccessful trajectories. Subsequent works, including CLIPO \citep{cui2026clipo} and ConSPO \citep{zhang2026revisiting}, further introduce explicit contrastive objectives based on InfoNCE \citep{oord2018representation}. However, existing approaches mainly operate at the trajectory outcome level, where trajectories within the same outcome category receive identical supervision. 
Our LACL-GUI addresses this limitation by introducing trajectory-level quality preferences into contrastive RLVR. 

\section{Methodology}

\subsection{Problem Formulation}
We formulate GUI task completion as a sequential decision-making problem under a Partially Observable Markov Decision Process (POMDP). Given a natural language instruction $q$, an agent interacts with the environment over multiple steps to complete the task. At each timestep $t$, the agent receives a visual observation $o_t \in \mathcal{O}$, represented as a screenshot of the interface, instead of accessing the full underlying state. To mitigate partial observability, the agent conditions on a finite history window 
$h_t = \{(o_{t-k}, z_{t-k}, a_{t-k}), \dots, (o_{t-1}, z_{t-1}, a_{t-1})\}$ 
of the most recent $k$ interactions, where $z_t$ denotes an intermediate reasoning trace and $a_t \in \mathcal{A}$ denotes a structured executable action (e.g., parameterized tool calls such as clicking at specific coordinates). A policy $\pi_\theta$, parameterized by a vision-language model, generates a reasoning-action pair:
\[
(z_t, a_t) \sim \pi_\theta(\cdot \mid o_t, h_t, q).
\]
In practice, $z_t$ serves as auxiliary reasoning tokens, while $a_t$ is the executable action applied to the environment. Executing action $a_t$ leads to the next observation $o_{t+1}$ according to the environment dynamics. An episode produces a trajectory
\[
\tau = \{(o_0, z_0, a_0), (o_1, z_1, a_1), \dots, (o_T, z_T, a_T)\},
\]
which terminates either when the task is completed or when a maximum horizon is reached.

We consider a sparse outcome-based reward setting, where supervision is provided only at the trajectory level:$
R(\tau) \in \{0,1\},
$ indicating task success or failure. All steps within a trajectory share the same outcome-level supervision, without any intermediate rewards for individual actions. The objective is to learn a policy $\pi_\theta$ that maximizes the expected success rate $
\theta^* = \arg\max_\theta \mathbb{E}_{\tau \sim \pi_\theta}[R(\tau)].$

\subsection{Architecture}


Following DART~\citep{li2025efficient}, we adopt an asynchronous training architecture for proposed LACL-GUI framework. The system consists of four modules: \textit{Rollout Service}, \textit{Data Manager}, \textit{Success Buffer}, and \textit{Trainer}. The \textit{Rollout Service} collects trajectories from multiple GUI environments in parallel, which are filtered and organized by the \textit{Data Manager}. Successful trajectories are stored in the \textit{Success Buffer} for reuse and divergence estimation, while filtered trajectory groups are sent to the \textit{Trainer} for policy optimization with the LACL-GUI objective. Updated parameters are periodically synchronized back to the rollout workers, forming an asynchronous training loop.

\paragraph{Rollout Service} 

Interacting with real-world environments to solve GUI tasks using vision-language models is time-consuming. To address this, we employ asynchronous rollout-wise sampling following DART~\citep{li2025efficient}, where each environment independently starts a new trajectory after completion without waiting for other workers. This design improves environment utilization and enables efficient trajectory collection. For each task instruction, we collect $N$ trajectories under the same initial condition, forming trajectory groups for contrastive optimization.

\paragraph{Data Manager}

The \textit{Data Manager} filters sampled trajectories to construct informative training groups. 
Since LACL-GUI relies on contrastive optimization between successful and failed trajectories, groups 
where all $N$ trajectories share the same outcome (i.e., all success or all failure) provide no 
discriminative supervision and are discarded. For each training iteration, the \textit{Data Manager} 
collects a batch of $M$ tasks whose trajectory groups contain both successful and failed samples, and 
then forwards them to the \textit{Trainer}. Following DART, policy updates are performed after all 
trajectories in the batch are collected, enabling meaningful within-task comparisons among sampled 
trajectories.

\paragraph{Success Buffer} 

Sparse rewards in GUI tasks often lead to insufficient successful trajectories during training. To 
improve sample efficiency, we maintain a \textit{Success Buffer} that stores successful trajectories 
for each task. The buffer serves two purposes: first, it provides additional positive samples when 
current rollouts contain limited successes; second, it provides reference successful trajectories for 
divergence-aware failure analysis. Specifically, for each failed trajectory ($R(\tau)=0$), we estimate its quality by comparing it with a successful reference trajectory from the same task. 
The reference is selected as the shortest successful trajectory from the union of the \textit{Success Buffer} and current-batch successful trajectories, while only current-batch successful trajectories are considered when the buffer is unavailable.

\paragraph{Trainer}

The \textit{Trainer} is decoupled from rollout workers and asynchronously optimizes the policy using filtered trajectory groups to eliminate blocking between data collection and policy optimization. After each update, the latest parameters are synchronized back to the rollout service, allowing trajectory collection and policy optimization to proceed concurrently.

\quad

Overall, this architecture provides an efficient infrastructure for collecting structured trajectory groups and maintaining high-quality successful references, which are essential for the trajectory-level contrastive optimization in LACL-GUI. We next introduce our proposed objective in detail.

\subsection{LACL-GUI Objective}

\subsubsection{Revisiting REAL Objective}
We first revisit the REAL \citep{zhai2026rewards} objective, which formulates RLVR from a classification perspective. 
Given an instruction $q$, a set of $N$ trajectories $\{\tau^1, \dots, \tau^N\}$ are sampled from the same policy. 
Each trajectory $\tau^j$ is assigned a binary reward $r^j \in \{0,1\}$ by a verifier indicating task success. 
Based on these labels, trajectories are partitioned into two disjoint sets:
\begin{equation*}
    \mathcal{T}^+ = \{\tau^j|r^j = 1\}, \quad \mathcal{T}^- = \{\tau^j|r^j = 0\}.
\end{equation*}

For each trajectory $\tau^j$ in both $\mathcal{T}^+$ and $\mathcal{T}^-$, REAL defines a length-normalized relative log probability score over its actions as \emph{logits}:
\begin{equation*}
    \bar{s}^j  = \frac{1}{|\tau^j|} \sum_{t=1}^{|\tau^j|} (\log{\frac{\pi_{\theta}(a_t|o_t, h_t,q) }{\pi_{old}(a_t|o_t, h_t,q)}}).
\end{equation*}
where $|\tau^j|$ denotes trajectory length, and $a_t$, $o_t$, and $h_t$ represent the action, observation, and history at step $t$, respectively. The score measures the relative change of trajectory probability under the updated policy compared to the old policy. Specifically, $\bar{s}^j>0$ indicates that the trajectory is promoted, whereas $\bar{s}^j<0$ indicates suppression. The binary reward determines the optimization direction by separating desirable and undesirable trajectories, while the logits reflect the magnitude of optimization updates.

Based on logits $\bar{s}^j$, REAL optimizes the following contrastive-style objective:
\begin{equation}
    \mathcal{L}_{REAL} = \log (1+\sum_{\mathcal{T^+}}e^{-\bar{s}^j/\lambda}) + \log  (1+\sum_{\mathcal{T^-}}e^{\bar{s}^j/\lambda}).\
    \label{equ::REALLoss}
\end{equation}
where $\lambda$ is a temperature parameter. This objective encourages positive trajectories and suppresses negative ones, with adaptive gradient allocation determined by $\bar{s}^j$. Specifically, under-optimized positive trajectories with smaller logits receive stronger updates, while over-confident positives are down-weighted; similarly, negative trajectories with larger logits receive stronger penalties. Such adaptive and bounded gradient allocation alleviates the gradient misalignment issues observed in prior RLVR methods \citep{zhai2026rewards}.

However, Eq.~\eqref{equ::REALLoss} only distinguishes trajectories at the task-outcome level and does not capture fine-grained differences within $\mathcal{T}^+$ or $\mathcal{T}^-$. 
Trajectories with identical outcomes can still exhibit substantially different qualities. For example, two successful GUI trajectories may complete the same task, while one requires only necessary interactions and the other contains redundant exploratory actions. Similarly, among failed trajectories, one may follow the correct procedure for most steps before making a final mistake, whereas another may diverge immediately from the desired behavior. Although these trajectories receive identical binary rewards, they provide different optimization signals. 
Since REAL does not explicitly model such intra-group preferences, it cannot encourage efficient successful executions or distinguish informative failures from fundamentally incorrect attempts. This limitation motivates our proposed LACL-GUI objective.


\subsubsection{Positive Group Logits Perturbation}

For fine-grained comparison within successful trajectories, we observe that GUI task execution naturally exhibits an efficiency preference. Among trajectories starting from the same initial state and achieving the same goal, shorter executions are generally preferred, as longer successful trajectories often contain redundant interactions caused by unnecessary exploration, incorrect actions followed by recovery, or repeated operations. 
For example, in GUI environments, an agent may perform unnecessary exploratory movements, repeatedly click irrelevant regions, or revisit previously completed steps before eventually reaching the correct state. Such inefficiency is particularly undesirable in GUI environments, where each interaction introduces additional inference and execution cost while increasing the possibility of future errors.

Based on this observation, we introduce a length-aware perturbation within the positive group in the logit space. 
Rather than modifying binary success supervision, the proposed perturbation preserves REAL's success preference while introducing additional efficiency-aware guidance among successful trajectories.

For the positive set $\mathcal{T}^+$, we define the softmax weighting based on logits $\bar{s}_{\theta}^j$ under current policy $\pi_{\theta}$:
\begin{equation}
p^+(\tau^j) = \frac{e^{-\bar{s}_\theta(\tau^j)/\lambda}}{\sum_{\tau^j \in \mathcal{T}^+} e^{-\bar{s}_\theta(\tau^j)/\lambda}}.
\end{equation}

We then compute the weighted mean trajectory length:
\begin{equation}
\mu_w = \sum_{\tau^j \in \mathcal{T}^+} p^+(\tau^j)\, |\tau^j|.
\end{equation}

Based on this efficiency preference, we perturb the positive logits as:
\begin{equation}
\tilde{s}_\theta(\tau^j) = \bar{s}_\theta(\tau^j) + \epsilon \left(|\tau^j| - \mu_w \right), \quad \tau^j \in \mathcal{T}^+.
\end{equation}
where $\epsilon>0$ is a small perturbation coefficient controlling the magnitude of the logit adjustment. This perturbation introduces a relative efficiency preference within the positive group. Under the exponential weighting $p(\tau^j)\propto\exp(-\tilde{s}_\theta(\tau^j)/\lambda)$, shorter successful trajectories receive relatively larger optimization weights, whereas longer trajectories are down-weighted. Therefore, LACL-GUI encourages efficient successful executions without changing the original success preference defined by REAL.

\subsubsection{Negative Group Logits Perturbation}
While REAL suppresses all failed trajectories uniformly, failures may contain different levels of learning value. 
In GUI tasks, a trajectory that follows successful behaviors for most steps before deviating is generally closer to task completion than one that fails at an early stage. We therefore introduce a \emph{divergence-aware perturbation} to distinguish failed trajectories according to their proximity to successful executions.

For each task, given a negative trajectory $\tau^j\in\mathcal{T}^{-}$, we compare it with a reference successful trajectory $\tau_{ref}^{+}$:
\begin{equation}
    \begin{aligned}
        \tau^j &= \{ (o_1^j, a_1^j), \dots , (o_{|\tau^j|}^j, a_{|\tau^j|}^j) \}; \\
        \tau^+_{ref} &= \{ (o_1^{ref}, a_1^{ref}), \dots , (o_{|\tau_{ref}^+|}^{ref}, a_{|\tau^+_{ref}|}^{ref}) \}.
    \end{aligned}
\end{equation}
The reference trajectory $\tau_{ref}^{+}$ is selected as the shortest successful trajectory among the success buffer and the current rollout group. Since trajectories of the same task share an identical initial state, their observations can be compared step by step to identify the first divergence point.

We compute the observation similarity 
$\mathrm{SIM}(o_t^{ref},o_t^j)$ 
at each step using Structural Similarity Index (SSIM) \citep{brunet2012mathematical}, following \citep{lin2026ui}. 
Given a threshold $\Theta$, the divergence length is defined as:
\begin{equation}
    d^j = \min \{t|\mathrm{SIM}(o_t^{ref}, o_t^j) < \Theta\} - 1.
\end{equation}
If no divergence is detected, we set $d^j=|\tau^j|$.

Based on the divergence length, we define the number of \emph{erroneous steps}
\begin{equation}
    E^j = |\tau^j| - d^j + 1.
\end{equation}
which measures the remaining execution length after deviating from the successful trajectory.

Analogous to the positive group, we compute the weighted average erroneous steps:
\begin{equation}
\begin{aligned}
    p^-(\tau^j) &= \frac{e^{\bar{s}_\theta(\tau^j)/\lambda}}{\sum_{\tau^j \in \mathcal{T}^-} e^{\bar{s}_\theta(\tau^j)/\lambda}};
\\
\mu_e &= \sum_{\tau^j \in \mathcal{T}^-} p^-(\tau^j)\, |E^j|.
\end{aligned}
\end{equation}

The perturbed logits are then given by
\begin{equation}
    \hat{s}_\theta(\tau^j) = \bar{s}_\theta(\tau^j) + \epsilon \left(E^j - \mu_e \right), \quad \tau^j \in \mathcal{T}^-.
\end{equation}

This perturbation introduces a \emph{relative error-severity preference} within the negative group. Failures with larger erroneous-step counts, corresponding to earlier divergence from successful executions, receive larger logits and therefore stronger suppression under the negative loss. In contrast, near-success trajectories with fewer post-divergence steps are penalized more gently. Compared with the uniform negative suppression in REAL, LACL-GUI preserves informative near-success failures while still reducing the influence of fundamentally incorrect execution patterns.



\begin{table*}[t]
\centering
\scriptsize
\setlength{\tabcolsep}{4pt}

\begin{tabular}{l c ccccccccc}
\toprule
\textbf{Model} & 
\textbf{Max Steps} & 
\textbf{gimp} & 
\textbf{calc} & 
\textbf{impress} & 
\textbf{writer} & 
\textbf{os} & 
\textbf{thunderbird} & 
\textbf{vlc} & 
\textbf{vs\_code} & 
\textbf{Overall} \\

\midrule

\multicolumn{11}{c}{\emph{Closed-source Models}}\\
\midrule

OpenAI CUA o3 \citep{openai2025o3o4mini}
& 100 & 38.5 & 10.6 & 10.6 & 30.4 & 62.5 & 26.7 & 39.2 & 39.1 & 27.8 \\

OpenAI CUA \citep{openai2025o3o4mini}
& 50 & 34.6 & 14.9 & 29.7 & 26.1 & 70.8 & 66.7 & 11.6 & 69.6 & 36.5 \\

Claude 3.7 Sonnet \citep{anthropic2025claude37}
& 50 & 38.5 & 31.9 & 36.1 & 43.5 & 50.0 & 53.3 & 23.5 & 56.5 & 40.1 \\

Seed1.5-VL-250717 \citep{guo2025seed1}
& 100 & 50.0 & 34.8 & 48.9 & 56.5 & 39.1 & 73.3 & 35.3 & 56.5 & 45.1 \\

Claude 4 Sonnet \citep{anthropic2025claude4}
& 50 & 50.0 & 31.9 & 46.7 & 60.9 & 45.8 & 73.3 & 41.2 & 60.9 & 47.3 \\

DeepMiner-Mano-7B \citep{fu2025mano}
& 100 & 69.2 & 27.7 & 42.5 & 56.5 & 50.0 & 73.3 & 35.3 & 78.3 & 47.9 \\

\midrule
\multicolumn{11}{c}{\emph{Open-source Models}}\\
\midrule

OpenCUA-7B \citep{wang2026opencua}
& 50 & 43.6 & 13.2 & 32.6 & 33.3 & 43.5 & 42.2 & 28.3 & 47.1 & 32.9 \\

UI-TARS-1.5-7B \citep{qin2025ui}
& 100 & 51.9 & 9.6 & 38.2 & 39.1 & 31.3 & 40.0 & 22.4 & 47.8 & 33.0 \\

UI-TARS-72B-dpo \citep{qin2025ui}
& 100 & 73.1 & 6.4 & 23.8 & 34.8 & 37.5 & 60.0 & 17.7 & 52.2 & 33.4 \\

OpenCUA-32B \citep{wang2026opencua}
& 100 & 66.7 & 18.4 & 37.6 & 36.2 & 55.1 & 46.7 & 33.3 & 63.3 & 41.6 \\

DART-GUI-7B-0924 \citep{li2025efficient}
& 30 & 80.8 & 14.9 & 44.7 & 47.8 & 54.2 & 66.7 & 27.2 & 69.6 & 46.7 \\

\midrule
\multicolumn{11}{c}{\emph{Baselines}}\\
\midrule

Qwen3-VL-4B-Thinking \citep{bai2025qwen3}
& 50 & 69.2 & 25.3 & 19.1 & 43.5 & 56.5 & 45.8 & 60.0 & 56.3 & 41.2 \\

Qwen3-VL-4B-GRPO
& 50 & 73.1 & 10.6 & 31.7 & 56.5 & 54.2 & 66.7 & 46.1 & 47.8 & 42.2 \\

Qwen3-VL-4B-REAL
& 50 & 76.9 & 17.0 & 27.5 & 52.2 & 50.0 & 66.7 & 66.4 & 47.8 & 43.0 \\

Qwen3-VL-8B-Thinking \citep{bai2025qwen3}
& 50 & 76.9 & 12.8 & 29.6 & 52.2 & 62.5 & 60.0 & 23.5 & 65.2 & 42.7 \\

Qwen3-VL-8B-GRPO
& 50 & 65.4 & 25.5 & 33.8 & 56.5 & 66.7 & 73.3 & 34.6 & 60.9 & 47.2 \\

Qwen3-VL-8B-REAL
& 50 & 76.9 & 21.3 & 33.8 & 69.6 & 66.7 & 60.0 & 41.2 & 47.8 & 47.3 \\

\midrule
\multicolumn{11}{c}{\emph{Proposed}}\\
\midrule

\textbf{LACL-GUI-4B}
& 50 & 69.2 & 21.3 & 31.7 & 52.2 & 58.3 & 66.7 & 41.2 & 56.5 & 44.5 \\

\textbf{LACL-GUI-8B}
& 50 & 76.9 & 27.7 & 36.0 & 65.2 & 70.8 & 66.7 & 47.1 & 47.8 & 50.0 \\

\bottomrule
\end{tabular}

\caption{Task success rate (\%) on the OSWorld benchmark across different application domains. Results are reported following the official OSWorld evaluation protocol. Qwen3-VL-GRPO and Qwen3-VL-REAL denote models trained with GRPO and REAL objectives, respectively, based on the corresponding Qwen3-VL backbones.}
\label{tab:main_results_domain}

\end{table*}

\subsubsection{Overall Loss}
Combining the positive and negative logit perturbations, the LACL-GUI objective is formulated as
\begin{equation}
\mathcal{L}_{\text{LACL-GUI}}
=~
\log\bigl(
1+
\sum_{\tau^j\in\mathcal{T}^{+}}
e^{-\tilde{s}_{\theta}(\tau^j)/\lambda}\bigr) + \log\bigl(
1+ \sum_{\tau^j\in\mathcal{T}^{-}}
e^{\hat{s}_{\theta}(\tau^j)/\lambda}\bigr),
\end{equation}
where $\tilde{s}_{\theta}(\tau)$ and $\hat{s}_{\theta}(\tau)$ denote the perturbed logits defined for the positive and negative groups, respectively.

In practice, the group-wise statistics $\mu_w$ and $\mu_e$ are detached from the computational graph using the stop-gradient operator. Consequently, gradients are propagated only through the perturbed logits, while the group statistics are treated as constants during optimization.

\begin{theorem}
    (First-order Invariance). Assume that group-wise statistics $\mu_w$ and $\mu_e$ hold as constants via the stop-gradient operation. Then LACL-GUI objective satisfies:
    \begin{equation}
        \mathcal{L}_{LACL-GUI} = \mathcal{L}_{REAL} + \mathcal{O}(\epsilon^2).
    \end{equation}
    \label{theorem::first}
\end{theorem}

Theorem~\ref{theorem::first} (See proof in Appendix) shows that LACL-GUI preserves the first-order optimization objective of REAL while introducing higher-order logit adjustments. Therefore, proposed perturbations maintain the original success-failure preference encoded by REAL at the first order, while reshaping gradient allocation and local optimization geometry within each outcome group. Consequently, LACL-GUI refines optimization without changing the underlying success-failure classification boundary. Unlike explicit length penalties that modify the optimization target, LACL-GUI introduces efficiency and error-severity preferences effectively through gradient modulation.  LACL-GUI augments trajectory-level classification with fine-grained within-group preferences. By exploiting relative efficiency and divergence information, LACL-GUI provides more informative optimization signals while retaining the stability benefits of classification-based RLVR.

\section{Experiments}

\textbf{Settings of Benchmark}: 
We evaluate LACL-GUI on \textbf{OSWorld}~\citep{xie2024osworld}, a comprehensive benchmark for evaluating MLLM-based GUI agents in realistic computer environments. Following \citep{wang2026rlanything}, we select 222 tasks from the original 369-task benchmark for training and evaluation (detailed settings are provided in the Appendix). This subset removes tasks affected by external instability, such as unreliable network conditions and corrupted files, which may introduce evaluation noise. The selected tasks cover four application domains: \textit{(i) Operating System (OS),  (ii) Office} (LibreOffice Calc, Impress, Writer), \textit{(iii) Daily-use software} (VLC Player, Thunderbird), and \textit{(iv) Professional tools} (VS Code, Gimp). This diversity ensures a comprehensive coverage of realistic user scenarios. We follow the standard OSWorld evaluation protocol~\citep{xie2024osworld}, where execution-based validation scripts determine task success and assign trajectory rewards $R(\tau)\in\{0,1\}$.

\noindent \textbf{Implementation}: We adopt Qwen3-VL-8B-Thinking and Qwen3-VL-4B-Thinking~\citep{bai2025qwen3} as backbone models for LACL-GUI and baselines training. Each trajectory is limited to a maximum of 50 interaction steps. At each iteration, we sample 6 tasks and collect $N=8$ trajectories per task after filtering by the \textit{Data Manager}. The temperature parameter $\lambda$ and perturbation coefficient $\epsilon$ are fixed to 1.0 and 0.03, respectively, while the history horizon of $h_t$ is set to 5 steps following DART~\citep{li2025efficient}. All models are trained with 8 NVIDIA H100 GPUs for up to 50 iterations. We report the best checkpoint selected according to evaluation performance across training iterations. 

\subsection{Main Results}

\paragraph{Overall Performance.}

Table~\ref{tab:main_results_domain} presents the performance comparison on the OSWorld benchmark across 8 applications. Unless otherwise specified, LACL-GUI models are trained with a similarity threshold $\Theta=0.85$. LACL-GUI achieves competitive performance against both open-source and closed-source baselines. In particular, LACL-GUI-8B achieves an overall success rate of 50.0\%, improving the backbone model Qwen3-VL-8B-Thinking by 7.3\%. It also surpasses several strong open-source and closed-source models, including DART-GUI-7B (46.7\%) and DeepMiner-Mano-7B (47.9\%). Moreover, under the same backbone architecture, LACL-GUI consistently outperforms RLVR baselines, achieving improvements over GRPO (47.2\%) and REAL (47.3\%). These results demonstrate the effectiveness of incorporating trajectory-level quality information into contrastive RLVR optimization.

\paragraph{Effect over RL Baselines.}
Compared with existing RL-based objectives, LACL-GUI consistently improves performance across different model scales. On the 8B backbone, LACL-GUI improves over REAL and GRPO by 2.7 and 2.8 percentage points, respectively. Similarly, on the 4B backbone, LACL-GUI achieves a higher success rate than REAL and GRPO (44.5\% vs. 43.0\%, 42.2\%). These gains indicate that explicitly modeling preferences within successful and failed trajectory groups provides more informative optimization signals than outcome-only supervision.

\paragraph{Domain-wise Improvements.}
LACL-GUI achieves substantial improvements across several application domains. In particular, LACL-GUI-8B obtains strong performance on \textit{OS} (70.8\%) and \textit{Calc} (27.7\%), while also achieving competitive results on \textit{Gimp} (76.9\%), \textit{Writer} (65.2\%), and \textit{Thunderbird} (66.7\%). These improvements demonstrate the effectiveness of LACL-GUI in handling diverse GUI scenarios.

\paragraph{Summary.}
Overall, LACL-GUI provides more informative trajectory-level supervision than existing RL objectives, leading to consistent improvements in GUI agent performance across varying scales and application domains.

\subsection{Ablation Studies}

\begin{table}[htbp]
\centering
\small
\setlength{\tabcolsep}{6pt}
\begin{tabular}{lcc}
\toprule
\textbf{Configuration} & 
\textbf{Success Rate (\%)} & 
\textbf{Gain (\%)} \\
\midrule
Prompt Baseline & 42.7 & -- \\
+ Decoupled Architecture & 45.4 & +2.7 \\
+ Success Buffer & 46.5 & +1.1 \\
+ Trajectory Filter & 47.2 & +0.7 \\
+ LACL-GUI Objective & \textbf{50.0} & \textbf{+2.8} \\
\bottomrule
\end{tabular}
\caption{Ablation study of different components in the LACL-GUI framework. 
Each row incrementally introduces one additional component over the previous configuration. 
The prompt-based baseline uses Qwen3-VL-8B-Thinking, while GRPO is adopted for all training configurations except the proposed LACL-GUI objective.}
\label{tab:component_ablation}
\end{table}

\subsubsection{Effect of Architecture Components}
We first investigate the contribution of each component in LACL-GUI by progressively adding the decoupled architecture, success buffer, trajectory filtering, and finally the LACL-GUI objective to the prompt-based Qwen3-VL-8B-Thinking baseline. Results are shown in Table~\ref{tab:component_ablation}.

The architecture components improve performance from 42.7\% to 47.2\%, demonstrating the benefits of efficient rollout management, successful trajectory reuse, and informative contrastive groups under sparse rewards. Introducing the LACL-GUI objective achieves the largest single gain (+2.8\%), further improving performance to 50.0\%. Since this improvement is obtained on top of the complete GRPO-based framework, the gain mainly comes from the proposed length-aware preference modeling and divergence-aware failure refinement rather than architectural changes alone.

\subsubsection{Effect of Length-aware Preference Modeling}

\begin{figure}[htbp]
    \centering
    \includegraphics[width=0.8\linewidth]{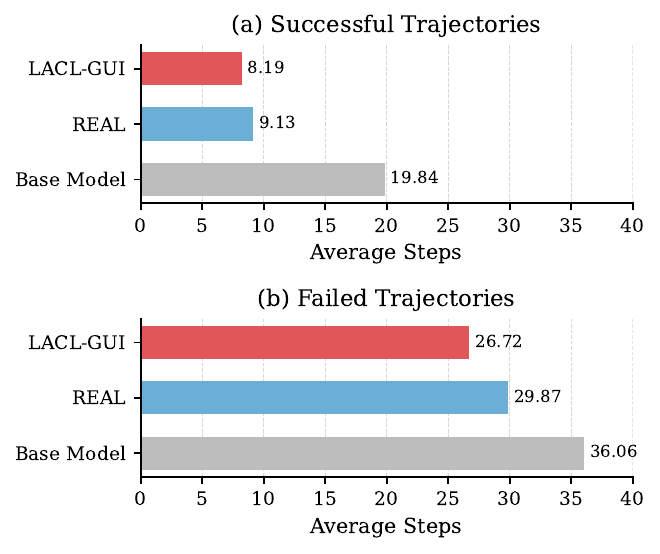}
    \caption{
    Effect of length-aware preference modeling in LACL-GUI. 
    We compare average trajectory lengths on tasks where all three methods achieve the same outcome. 
    }
    \label{fig:length_ablation}
\end{figure}

We further evaluate whether LACL-GUI objective improves trajectory efficiency beyond outcome-based optimization. To isolate trajectory preference from success-rate differences, we only analyze tasks where all methods achieve the same outcome. Results are shown in Figure~\ref{fig:length_ablation}.

For successful trajectories, LACL-GUI reduces the average length from 9.13 steps (REAL) to 8.19 steps, achieving a 10.3\% reduction, while the base model requires 19.84 steps. This verifies our motivation that, under the same initial state and goal, shorter successful executions generally indicate fewer redundant interactions and higher-quality solutions. For failed trajectories, LACL-GUI similarly reduces the average length from 29.87 to 26.72 steps (10.5\% reduction). Unlike global length penalties that may favor short failures over successful behaviors, LACL-GUI performs length modulation within the negative group through divergence-aware erroneous-step modeling, suppressing prolonged ineffective failed trajectories while preserving near-success trajectories as informative signals.

\subsubsection{Sensitivity Analysis of Divergence Estimation}

\begin{figure}[t]
    \centering
    \includegraphics[width=0.85\linewidth]{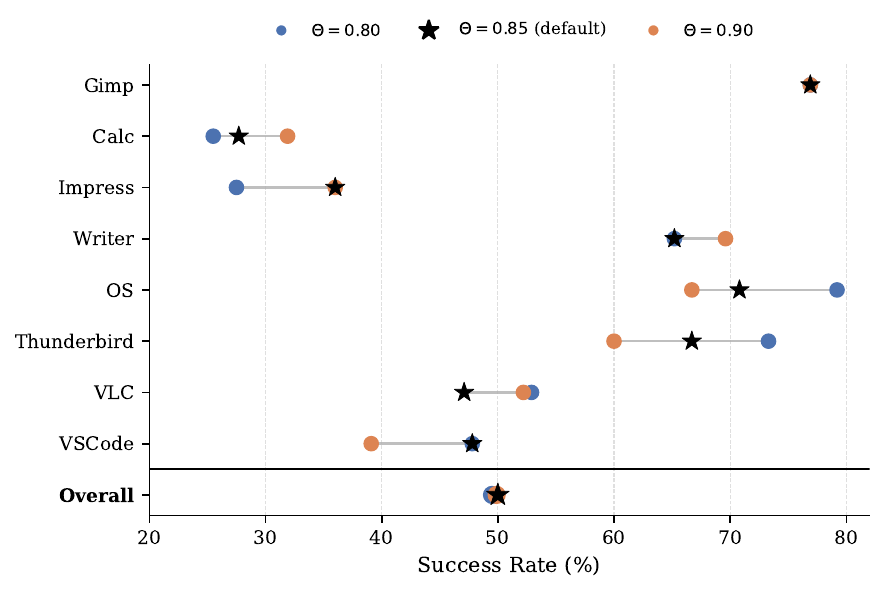}
    \caption{
   Sensitivity analysis of the similarity threshold $\Theta$ for divergence estimation in LACL-GUI. The default setting ($\Theta=0.85$) is highlighted, and results are reported across different domains.
   }
    \label{fig:theta_sensitivity}
\end{figure}

The divergence estimation in LACL-GUI relies on a similarity threshold $\Theta$ to determine the deviation point between failed and successful trajectories. We evaluate its sensitivity by varying $\Theta$ from 0.80 to 0.90 and report the resulting performance across different domains in Figure~\ref{fig:theta_sensitivity}.

LACL-GUI maintains stable performance across different threshold settings. The overall success rate changes only marginally (49.5\%, 50.0\%, and 49.9\%), despite larger variations in individual domains. This demonstrates that the effectiveness of divergence-aware failure refinement does not rely on precise threshold selection. Instead, the relative ordering of failed trajectories according to their divergence behaviors remains stable, enabling reliable identification of near-successful failures.



\section{Conclusion}

In this work, we present LACL-GUI, a length-aware contrastive learning framework for training GUI agents under sparse trajectory-level supervision.  LACL-GUI introduces fine-grained preference modeling within both successful and failed trajectory groups. 
It encourages efficient successful behaviors through length-aware positive trajectory refinement and distinguishes informative near-successful failures from early-diverging attempts through divergence-aware failure modeling. To support optimization in long-horizon GUI environments,  a decoupled asynchronous training architecture is employed with trajectory filtering and a successful trajectory buffer, enabling efficient rollouts collection and providing successful references for divergence estimation. Extensive experiments on the OSWorld benchmark demonstrate that LACL-GUI consistently improves GUI task completion performance over existing RL-based baselines. 
Future work will explore more adaptive trajectory quality estimation strategies and extend to broader interactive environments. 




\bibliography{iclr2025_conference}
\bibliographystyle{iclr2025_conference}

\appendix
\section{Appendix}

\subsection{Proof of Theorem 1}
\noindent We first prove Lemma \ref{lemma:zero_mean}:

\begin{lemma}[Weighted Zero-Mean Property]
\label{lemma:zero_mean}
For the proposed logit perturbations in LACL-GUI, the weighted perturbations satisfy
\begin{equation}
\sum_{\tau^j\in\mathcal{T}^{+}}
e^{-\bar{s}_{\theta}(\tau^j)/\lambda}
\bigl(|\tau^j|-\mu_w\bigr)=0
\end{equation}
and
\begin{equation}
\sum_{\tau^j\in\mathcal{T}^{-}}
e^{\bar{s}_{\theta}(\tau^j)/\lambda}
\bigl(E^j-\mu_e\bigr)=0.
\end{equation}
\end{lemma}

\begin{proof}
For the positive group, the weighted mean trajectory length is defined as
\begin{equation}
\mu_w=
\sum_{\tau^j\in\mathcal{T}^{+}}
p^{+}(\tau^j)\,|\tau^j|,
\end{equation}
where
\begin{equation}
p^{+}(\tau^j)
=
\frac{
e^{-\bar{s}_{\theta}(\tau^j)/\lambda}
}{
\sum_{\tau^j\in\mathcal{T}^{+}}
e^{-\bar{s}_{\theta}(\tau^j)/\lambda}
}.
\end{equation}
Therefore,
\begin{align*}
&\sum_{\tau^j\in\mathcal{T}^{+}}
e^{-\bar{s}_{\theta}(\tau^j)/\lambda}
\bigl(|\tau^j|-\mu_w\bigr)
\\
&\quad=
\sum_{\tau^j\in\mathcal{T}^{+}}
e^{-\bar{s}_{\theta}(\tau^j)/\lambda}\,|\tau^j|
-
\mu_w
\sum_{\tau^j\in\mathcal{T}^{+}}
e^{-\bar{s}_{\theta}(\tau^j)/\lambda}
\\
&\quad=
\mu_w
\sum_{\tau^j\in\mathcal{T}^{+}}
e^{-\bar{s}_{\theta}(\tau^j)/\lambda}
-
\mu_w
\sum_{\tau^j\in\mathcal{T}^{+}}
e^{-\bar{s}_{\theta}(\tau^j)/\lambda}
\\
&\quad=0.
\end{align*}
The negative-group identity follows by the same argument with the corresponding weighted mean \(\mu_e\) and weights \(e^{\bar{s}_{\theta}(\tau^j)/\lambda}\).
\end{proof}

\textbf{Theorem 1.} \textit{(First-order Invariance). Assume that group-wise statistics $\mu_w$ and $\mu_e$ hold as constants via the stop-gradient operation. Then LACL-GUI objective satisfies:}
    \begin{equation}
        \mathcal{L}_{LACL-GUI} = \mathcal{L}_{REAL} + \mathcal{O}(\epsilon^2).
    \end{equation}

\begin{proof}
We denote the positive-group loss of REAL and LACL-GUI by $\mathcal{L}^{+}_{\mathrm{REAL}}$ and $\tilde{\mathcal{L}}^{+}$, respectively, and the negative-group loss of REAL and LACL-GUI by $\mathcal{L}^{-}_{\mathrm{REAL}}$ and $\tilde{\mathcal{L}}^{-}$, respectively.

We first consider the positive group. The original REAL positive loss (in log-sum-exp form) is
\begin{equation}
\mathcal{L}^{+}_{\mathrm{REAL}}
=
\log\Bigl(
1+
\sum_{\tau^j\in\mathcal{T}^{+}}
e^{-\bar{s}_{\theta}(\tau^j)/\lambda}
\Bigr).
\end{equation}

After perturbation it becomes
\begin{equation}
\tilde{\mathcal{L}}^{+}
=
\log\Bigl(
1+
\sum_{\tau^j\in\mathcal{T}^{+}}
e^{-(\bar{s}_{\theta}(\tau^j)+\epsilon g(\tau^j))/\lambda}
\Bigr).
\end{equation}
where $g(\tau^j)) = |\tau^j| - \mu_w$.

\noindent Using the first-order Taylor expansion
\begin{equation}
e^{-\epsilon g(\tau^j)/\lambda}
=
1 - \frac{\epsilon g(\tau^j)}{\lambda} + \mathcal{O}(\epsilon^{2}),
\end{equation}
we can expand the sum as
\begin{align*}
&\sum_{\tau^j\in\mathcal{T}^{+}}
e^{-(\bar{s}_{\theta}(\tau^j)+\epsilon g(\tau^j))/\lambda}
\\
&\quad=
\sum_{\tau^j\in\mathcal{T}^{+}}
e^{-\bar{s}_{\theta}(\tau^j)/\lambda}
\cdot
e^{-\epsilon g(\tau^j)/\lambda}
\\
&\quad=
\sum_{\tau^j\in\mathcal{T}^{+}}
e^{-\bar{s}_{\theta}(\tau^j)/\lambda}
\left(
1 - \frac{\epsilon g(\tau^j)}{\lambda} + \mathcal{O}(\epsilon^{2})
\right)
\\
&\quad=
\sum_{\tau^j\in\mathcal{T}^{+}}
e^{-\bar{s}_{\theta}(\tau^j)/\lambda}
-
\frac{\epsilon}{\lambda}
\sum_{\tau^j\in\mathcal{T}^{+}}
e^{-\bar{s}_{\theta}(\tau^j)/\lambda}\,g(\tau^j)
\\& \qquad +\mathcal{O}(\epsilon^{2}).
\end{align*}

The key observation is that the coefficient of the first-order term vanishes.
By Lemma~\ref{lemma:zero_mean}, we have the weighted zero-mean property
\begin{equation}
\sum_{\tau^j\in\mathcal{T}^{+}}
e^{-\bar{s}_{\theta}(\tau^j)/\lambda}\,g(\tau^j)
=0,
\end{equation}

Consequently, the linear term in \(\epsilon\) is exactly zero, and we obtain
\begin{equation}
\sum_{\tau^j\in\mathcal{T}^{+}}
e^{-(\bar{s}_{\theta}(\tau^j)+\epsilon g(\tau^j))/\lambda}
=
\sum_{\tau^j\in\mathcal{T}^{+}}
e^{-\bar{s}_{\theta}(\tau^j)/\lambda}
+
\mathcal{O}(\epsilon^{2}).
\end{equation}
Substituting this back into the log-sum-exp loss immediately yields
\begin{equation}
\tilde{\mathcal{L}}^{+}
=
\mathcal{L}^{+}_{\mathrm{REAL}}
+\mathcal{O}(\epsilon^{2}).
\end{equation}

We now turn to the negative group. The original REAL negative loss is
\begin{equation}
\mathcal{L}^{-}_{\mathrm{REAL}}
=
\log\Bigl(
1+
\sum_{\tau^j\in\mathcal{T}^{-}}
e^{\bar{s}_{\theta}(\tau^j)/\lambda}
\Bigr).
\end{equation}
After applying the perturbation \(\hat{s}_{\theta}(\tau^j)=\bar{s}_{\theta}(\tau^j)+\epsilon y(\tau^j)\), where \(y(\tau^j)=E(\tau^j)-\mu_e\), it becomes
\begin{equation}
\tilde{\mathcal{L}}^{-}
=
\log\Bigl(
1+
\sum_{\tau^j\in\mathcal{T}^{-}}
e^{(\bar{s}_{\theta}(\tau^j)+\epsilon y(\tau^j))/\lambda}
\Bigr).
\end{equation}
Using the first-order expansion
\begin{equation}
e^{\epsilon y(\tau^j)/\lambda}
=
1+\frac{\epsilon y(\tau^j)}{\lambda}+\mathcal{O}(\epsilon^{2})
\end{equation}
and following the same algebraic steps as in the positive group, we obtain
\begin{align*}
&\sum_{\tau^j\in\mathcal{T}^{-}}
e^{(\bar{s}_{\theta}(\tau^j)+\epsilon y(\tau^j))/\lambda}
\\
&\quad=
\sum_{\tau^j\in\mathcal{T}^{-}}
e^{\bar{s}_{\theta}(\tau^j)/\lambda}
+
\frac{\epsilon}{\lambda}
\sum_{\tau^j\in\mathcal{T}^{-}}
e^{\bar{s}_{\theta}(\tau^j)/\lambda}\,y(\tau^j)
+\mathcal{O}(\epsilon^{2}).
\end{align*}
By the weighted zero-mean property in Lemma~\ref{lemma:zero_mean} , the coefficient of the linear term vanishes. Therefore,
\begin{equation}
\tilde{\mathcal{L}}^{-}
=
\mathcal{L}^{-}_{\mathrm{REAL}}
+\mathcal{O}(\epsilon^{2}).
\end{equation}

Combining the results for both groups, we conclude that
\begin{equation}
\mathcal{L}_{\mathrm{LACL\text{-}GUI}}
=
\mathcal{L}_{\mathrm{REAL}}
+\mathcal{O}(\epsilon^{2}).
\end{equation}

Thus the proposed perturbations preserve the REAL objective up to first order while redistributing optimization emphasis in higher-order within each outcome group.
\end{proof}

\subsection{Experiment Details}

\subsubsection{OSWorld Benchmark Settings}
We evaluate on a subset of 222 tasks selected from the full OSWorld benchmark of 369 tasks. 
All experiments are conducted in Baidu's RLE virtual machine environment. We exclude \textit{Chrome} and \textit{Multi-app} tasks.
This task filtering is motivated by the need for stable training and evaluation conditions. 
First, \textit{Chrome} tasks require reliable external network connectivity, which is unstable in our RLE environment and frequently leads to non-reproducible outcomes. 
Second, \textit{Multi-app} tasks often depend on external asset files hosted on the official huggingface cache (for example, task \texttt{3f05f3b9-29ba-4b6b-95aa-2204697ffc06}, which requires the audio file ``Zhou Xuan - Nights in Shanghai.mp3''). 
These assets were frequently missing or incomplete in our setup, resulting in systematic environment failures rather than genuine agent errors. 
Both issues introduce significant instability into both training and testing, and the corresponding tasks are therefore excluded.

After removing the affected \textit{Chrome} and \textit{Multi-app} tasks, the remaining 222 tasks span eight application domains. 
Their distribution is summarized in Table~\ref{tab:task_distribution}.

\begin{table}[h]
\centering
\small
\begin{tabular}{lc}
\toprule
Domain & Number of Tasks \\
\midrule
Gimp & 26 \\
LibreOffice Calc & 47 \\
LibreOffice Impress & 47 \\
LibreOffice Writer & 23 \\
OS & 24 \\
Thunderbird & 15 \\
VLC & 17 \\
VS Code & 23 \\
\midrule
Total & 222 \\
\bottomrule
\end{tabular}
\caption{Distribution of the 222 selected OSWorld tasks across application domains.}
\label{tab:task_distribution}
\end{table}

\subsubsection{Effect of Length-aware Preference Modeling}

To further examine the effect of length-aware preference modeling, we conduct a fine-grained analysis of average trajectory lengths on tasks where the base model Qwen3-VL-8B-Thinking, REAL, and LACL-GUI all achieve the same outcome. This isolates the impact of trajectory preference from differences in task success rate. Results are reported separately for successful and failed trajectories.

Table~\ref{tab:success_length} presents the average trajectory lengths on the 70 tasks that are successfully completed by all three methods. LACL-GUI reduces the overall average length from 9.13 steps (REAL) to 8.19 steps (a 10.3\% reduction). The reduction holds in the majority of application domains (\textit{Gimp}, \textit{LibreOffice Calc}, \textit{OS}, \textit{Thunderbird}, and \textit{VS Code}), although the effect is not uniform: \textit{LibreOffice Writer} shows a slight increase and \textit{VLC} exhibits a clear increase. These results largely support the design motivation that, under identical initial states and goals, shorter successful executions generally correspond to higher-quality solutions with fewer redundant interactions.

Table~\ref{tab:fail_length} reports the corresponding statistics on the 93 tasks where all three methods fail. LACL-GUI again yields a shorter average trajectory (26.72 steps) compared with REAL (29.87 steps), corresponding to a 10.5\% reduction. Length reduction is observed in six out of eight domains, with particularly large improvements in \textit{Gimp} and \textit{OS}. Exceptions occur in \textit{LibreOffice Writer} and \textit{VS Code}, where LACL-GUI produces longer failed trajectories. Because length modulation is applied only within the negative group via divergence-aware erroneous-step modeling, LACL-GUI tends to suppress prolonged ineffective failures while still preserving near-success trajectories as informative learning signals.

\begin{table}[t]
\centering
\small
\begin{tabular}{lcccc}
\toprule
Domain & Base & LACL-GUI & REAL & Rel.\ to REAL \\
\midrule
Gimp & 36.46 & 12.08 & 13.38 & $-9.7\%$ \\
Calc & 17.50 & 10.17 & 16.00 & $-36.4\%$ \\
Impress & 4.25 & 4.25 & 4.25 & $0.0\%$ \\
Writer & 10.73 & 5.36 & 5.27 & $+1.7\%$ \\
OS & 21.67 & 4.25 & 5.50 & $-22.7\%$ \\
Thunderbird & 7.86 & 6.71 & 7.29 & $-8.0\%$ \\
VLC & 16.75 & 10.75 & 5.00 & $+115.0\%$ \\
VS Code & 30.67 & 13.44 & 15.56 & $-13.6\%$ \\
\midrule
Average & 19.84 & 8.19 & 9.13 & $-10.3\%$ \\
\bottomrule
\end{tabular}
\caption{Average trajectory lengths on the 70 tasks successfully completed by all three methods. The last column shows the relative change of LACL-GUI with respect to REAL.}
\label{tab:success_length}
\end{table}

\begin{table}[t]
\centering
\small
\begin{tabular}{lcccc}
\toprule
Domain & Base & LACL-GUI & REAL & Rel.\ to REAL \\
\midrule
Gimp & 33.80 & 5.20 & 25.00 & $-79.2\%$ \\
Calc & 39.91 & 30.12 & 33.97 & $-11.3\%$ \\
Impress & 30.48 & 27.86 & 31.24 & $-10.8\%$ \\
Writer & 37.83 & 46.67 & 42.67 & $+9.4\%$ \\
OS & 35.67 & 10.67 & 22.67 & $-52.9\%$ \\
Thunderbird & 50.00 & 26.75 & 27.75 & $-3.6\%$ \\
VLC & 33.14 & 13.14 & 15.57 & $-15.6\%$ \\
VS Code & 36.86 & 25.14 & 16.57 & $+51.7\%$ \\
\midrule
Average & 36.06 & 26.72 & 29.87 & $-10.5\%$ \\
\bottomrule
\end{tabular}
\caption{Average trajectory lengths on the 93 tasks where all three methods fail. The last column shows the relative change of LACL-GUI with respect to REAL.}
\label{tab:fail_length}
\end{table}

\subsection{Real-world Trajectories}

To provide an intuitive illustration of why length-aware preference within successful trajectories and divergence-aware modeling within failed trajectories are beneficial, we present concrete trajectory examples from the same \textit{Gimp} task:
\begin{quote}
\textit{Please help change GIMP's theme from dark to light.}
\end{quote}
(task ID: \\
\texttt{7767eef2-56a3-4cea-8c9f-48c070c7d65b}).
This task requires navigating \textit{Gimp}'s preference menus to switch the interface theme. 
Although the goal is relatively straightforward, different rollouts can exhibit substantial variation in efficiency and in how closely failed attempts approach a successful solution. 
We therefore examine both successful and failed trajectories produced for this task to highlight the fine-grained quality differences that outcome-level supervision alone cannot capture.

\paragraph{Positive Group: Efficient vs.\ Redundant Successful Trajectories.}

Although both trajectories receive the same success reward, successful executions can exhibit substantially different levels of efficiency. 
We illustrate this phenomenon using the \textit{Gimp} theme-switching task. 
The following example compares a concise successful trajectory generated by LACL-GUI-8B (6 steps) with a longer successful trajectory generated by REAL (27 steps).

The 6-step trajectory follows the optimal interaction path: it directly opens the Edit menu, enters Preferences, selects the Theme option, switches to Light mode, and confirms the change. 
No exploratory or recovery actions are involved.

\begin{tcolorbox}[
  colback=gray!8,
  colframe=gray!60,
  boxrule=0.4pt,
  arc=1pt,
  left=4pt,
  right=4pt,
  top=3pt,
  bottom=3pt,
  fontupper=\small\ttfamily
]
Step 1: Click on the "Edit" menu in the top menu bar.\\
\quad pyautogui.click(124, 72)\\[2pt]

Step 2: Click on the "Preferences" option in the Edit menu.\\
\quad pyautogui.click(180, 640)\\[2pt]

Step 3: Click on the "Theme" option under the Interface section.\\
\quad pyautogui.click(157, 667)\\[2pt]

Step 4: Select the "Light" theme option.\\
\quad pyautogui.click(363, 622)\\[2pt]

Step 5: Click the "OK" button to apply the theme.\\
\quad pyautogui.click(832, 867)\\[2pt]

Step 6: Terminate (task completed successfully).\\
\quad DONE
\end{tcolorbox}

\begin{figure}[H]
\centering
\includegraphics[width=0.92\columnwidth]{Figures/Trajectories/success_6steps_short/short_1.png}
\caption{LACL-GUI Step 1 (Success)}
\end{figure}

\begin{figure}[H]
\centering
\includegraphics[width=0.92\columnwidth]{Figures/Trajectories/success_6steps_short/short_2.png}
\caption{LACL-GUI Step 2 (Success)}
\end{figure}

\begin{figure}[H]
\centering
\includegraphics[width=0.92\columnwidth]{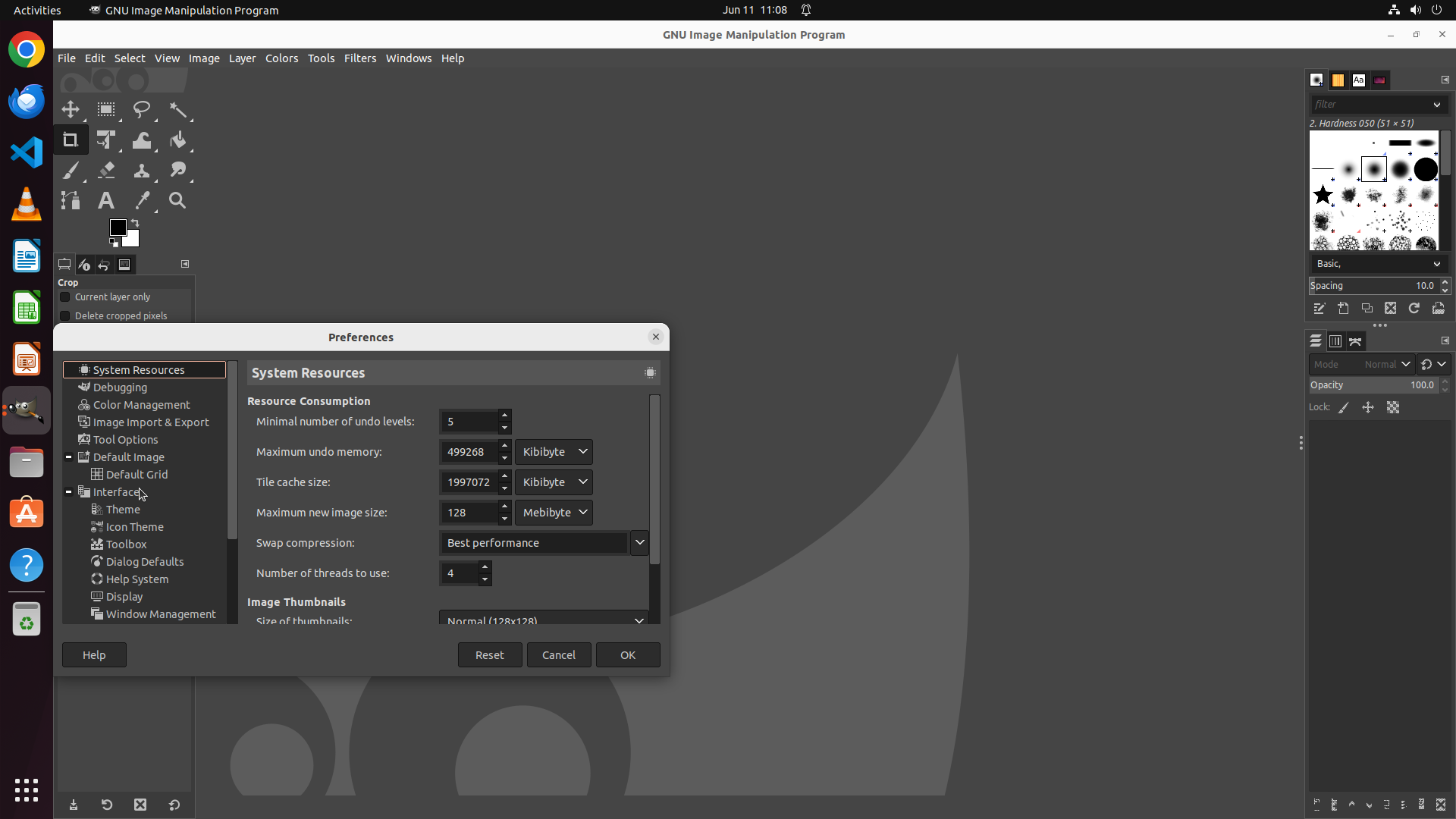}
\caption{LACL-GUI Step 3 (Success)}
\end{figure}

\begin{figure}[H]
\centering
\includegraphics[width=0.92\columnwidth]{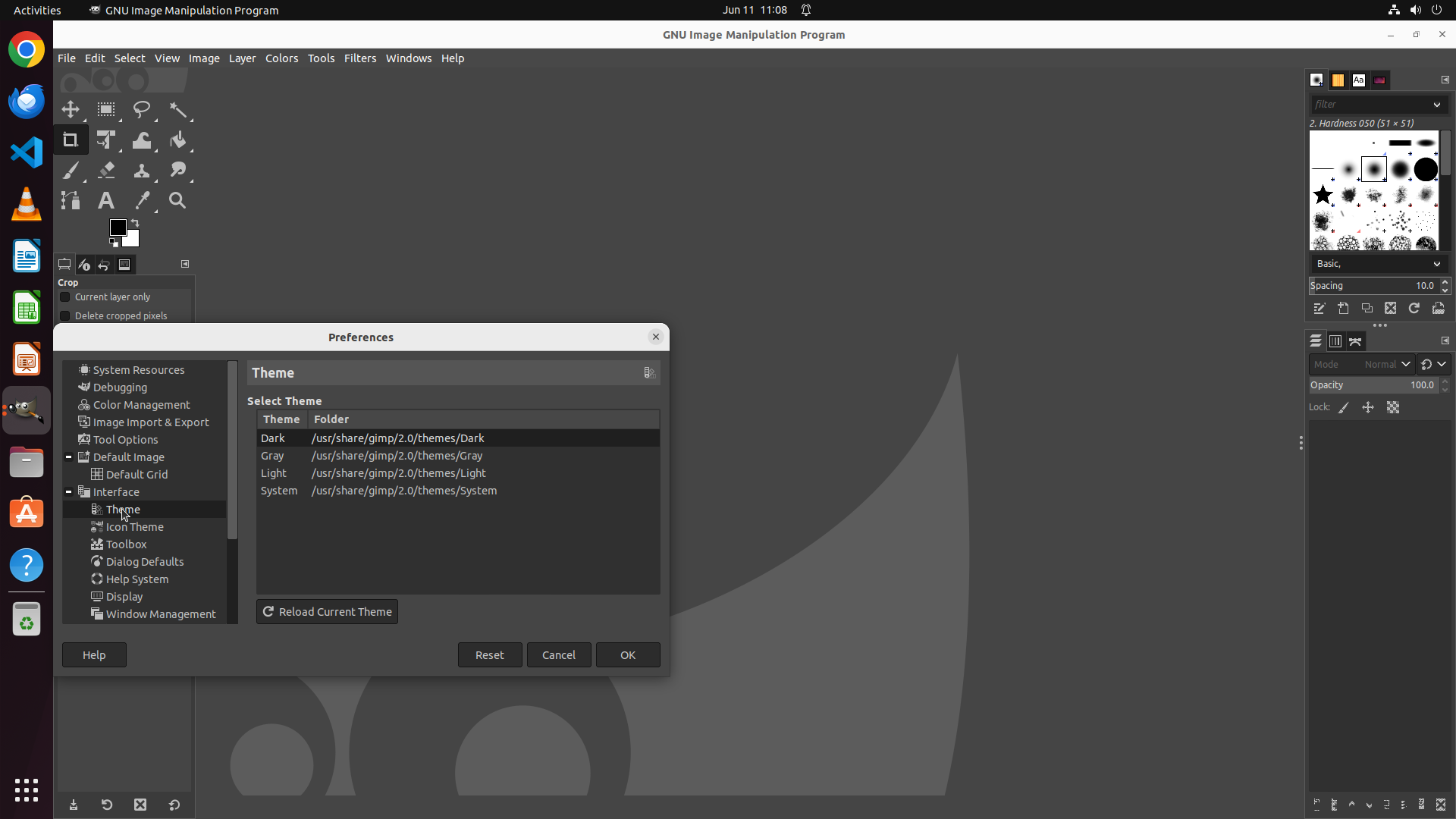}
\caption{LACL-GUI Step 4 (Success)}
\end{figure}

\begin{figure}[H]
\centering
\includegraphics[width=0.92\columnwidth]{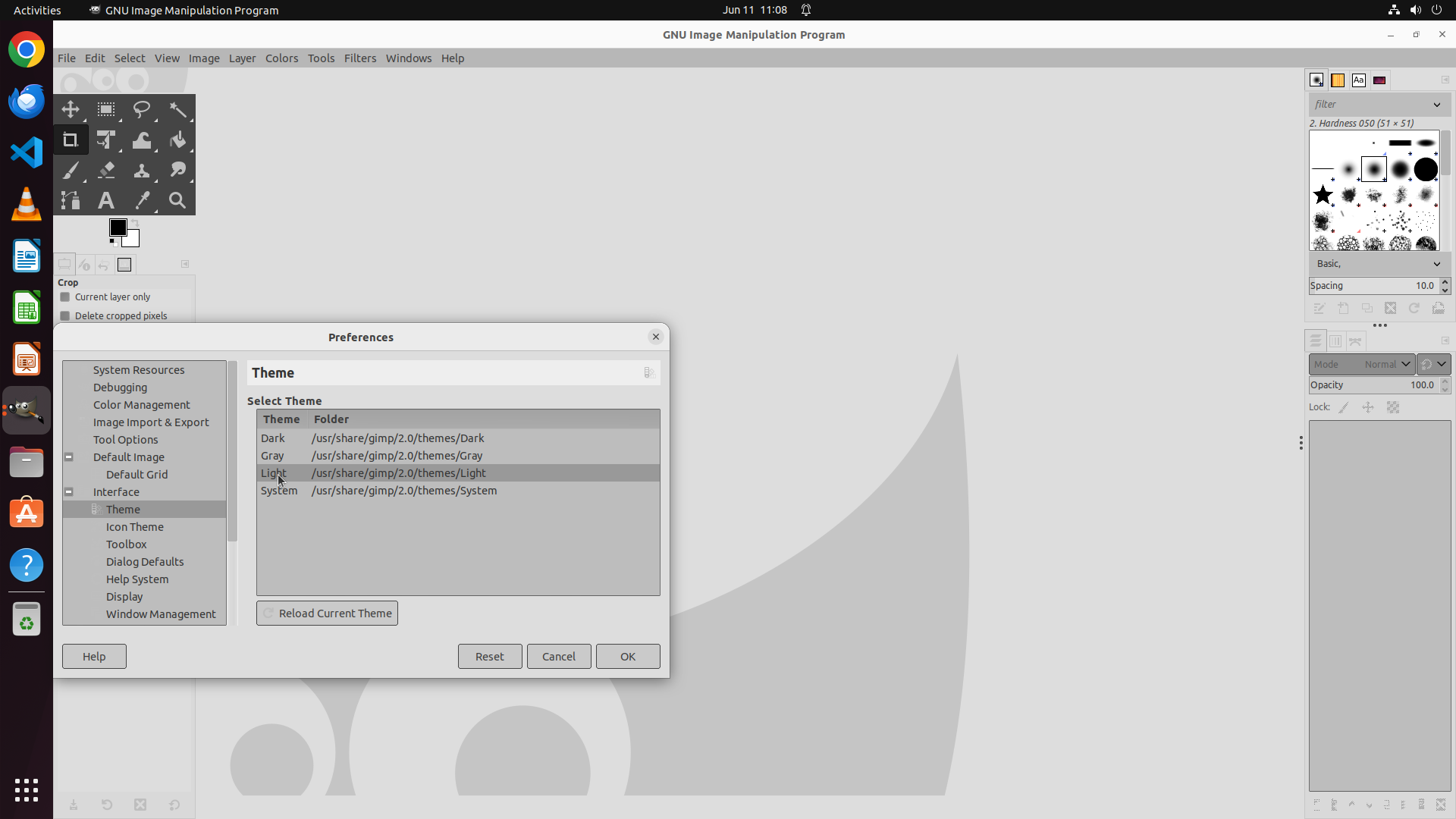}
\caption{LACL-GUI Step 5 (Success)}
\end{figure}

\begin{figure}[H]
\centering
\includegraphics[width=0.92\columnwidth]{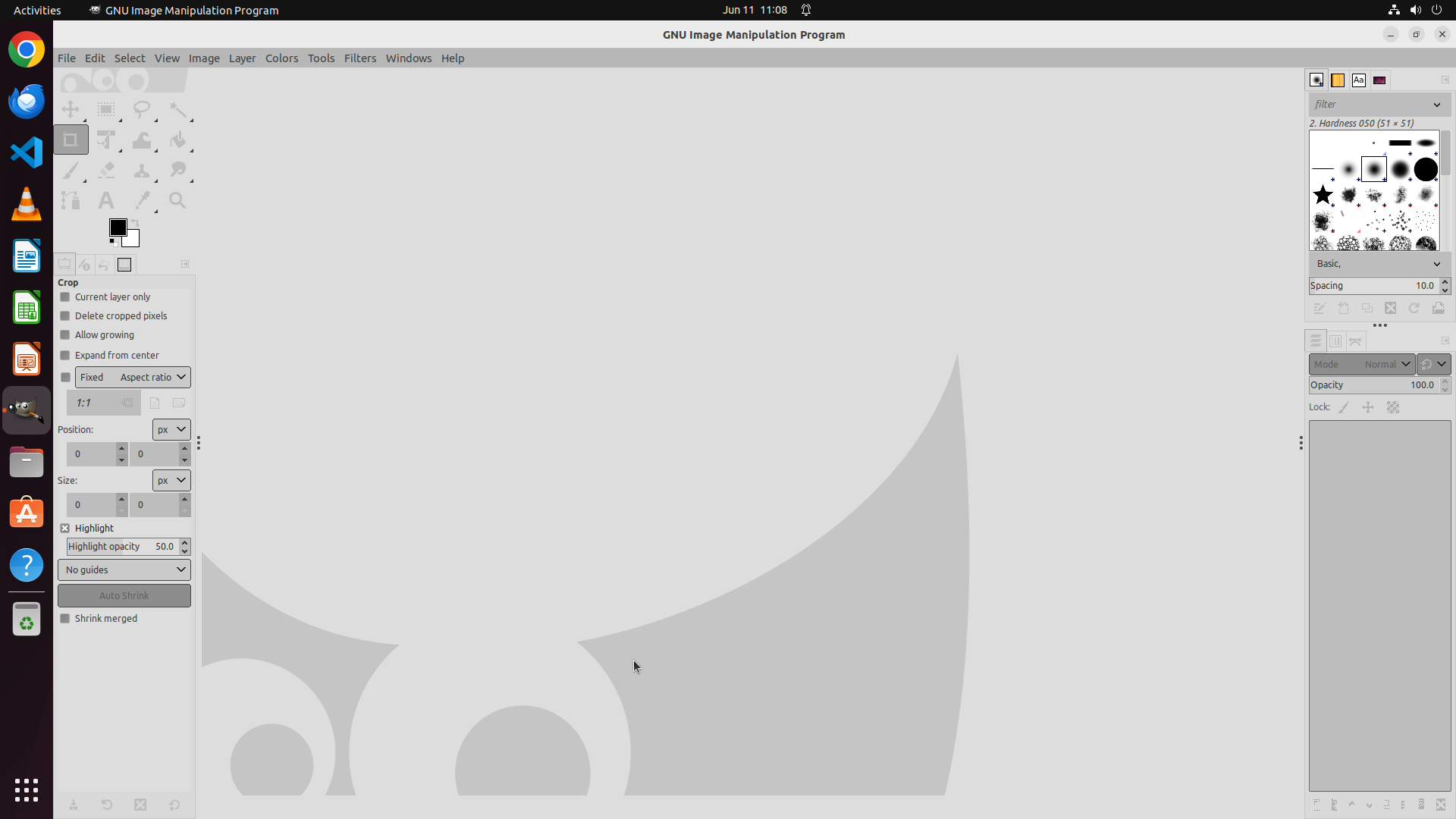}
\caption{LACL-GUI Step 6 (Success)}
\end{figure}

In contrast, the 27-step trajectory produced by Qwen3-VL-8B-REAL also successfully completes the task but contains extensive unnecessary interactions:

\begin{tcolorbox}[
  colback=gray!8,
  colframe=gray!60,
  boxrule=0.4pt,
  arc=1pt,
  left=4pt,
  right=4pt,
  top=3pt,
  bottom=3pt,
  fontupper=\small\ttfamily
]
Step 1-4: Navigate to the Theme panel correctly.\\[2pt]

Step 5-15: Repeatedly click the "Theme" item in the left sidebar.\\
\quad pyautogui.click(163, 667) \quad \text{(repeated)}\\[2pt]

Step 16: Click the "Theme" item again.\\[2pt]

Step 17-24: Continue clicking the same screen location after the UI state has changed.\\
\quad pyautogui.click(167, 667) \quad \text{(repeated)}\\[2pt]

Step 25: Select the "Light" theme option.\\
\quad pyautogui.click(359, 620)\\[2pt]

Step 26: Click the "OK" button to apply the change.\\
\quad pyautogui.click(830, 859)\\[2pt]

Step 27: Terminate (task completed successfully).\\
\quad DONE
\end{tcolorbox}

The long trajectory reveals two types of inefficiency. 
First, from Step 5 to Step 15, the agent repeatedly executes a valid action without making progress, producing redundant interactions. 
Second, from Step 17 to Step 24, the agent continues clicking the previous target location even after the interface state has changed, resulting in erroneous actions caused by outdated state understanding.

These additional steps do not contribute to task completion and are not desirable behaviors to reinforce, despite the trajectory receiving the same binary success reward. 
Therefore, length-aware positive preference in LACL-GUI assigns stronger preference to concise successful trajectories, encouraging the policy to learn efficient solutions rather than merely optimizing for eventual task completion.

The screenshots below highlight representative failure patterns in the long successful trajectory. 
Step 5 demonstrates redundant interactions where the agent repeatedly selects an already active option, while Step 17 illustrates incorrect repeated clicking after the UI state has changed.

\begin{figure}[H]
\centering
\includegraphics[width=0.92\columnwidth]{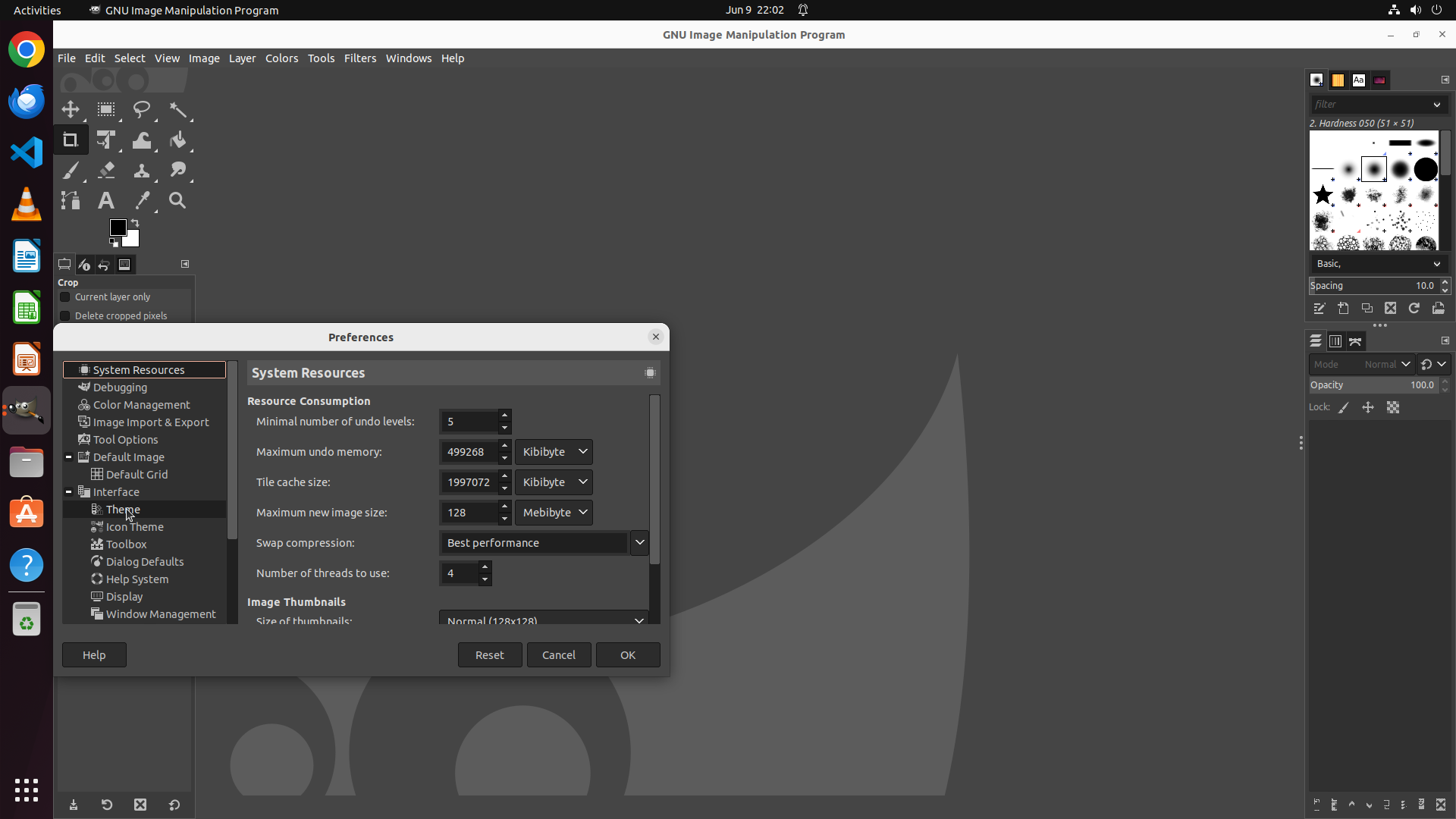}
\caption{A redundant interaction in the long successful trajectory (Step 5).}
\end{figure}

\begin{figure}[H]
\centering
\includegraphics[width=0.92\columnwidth]{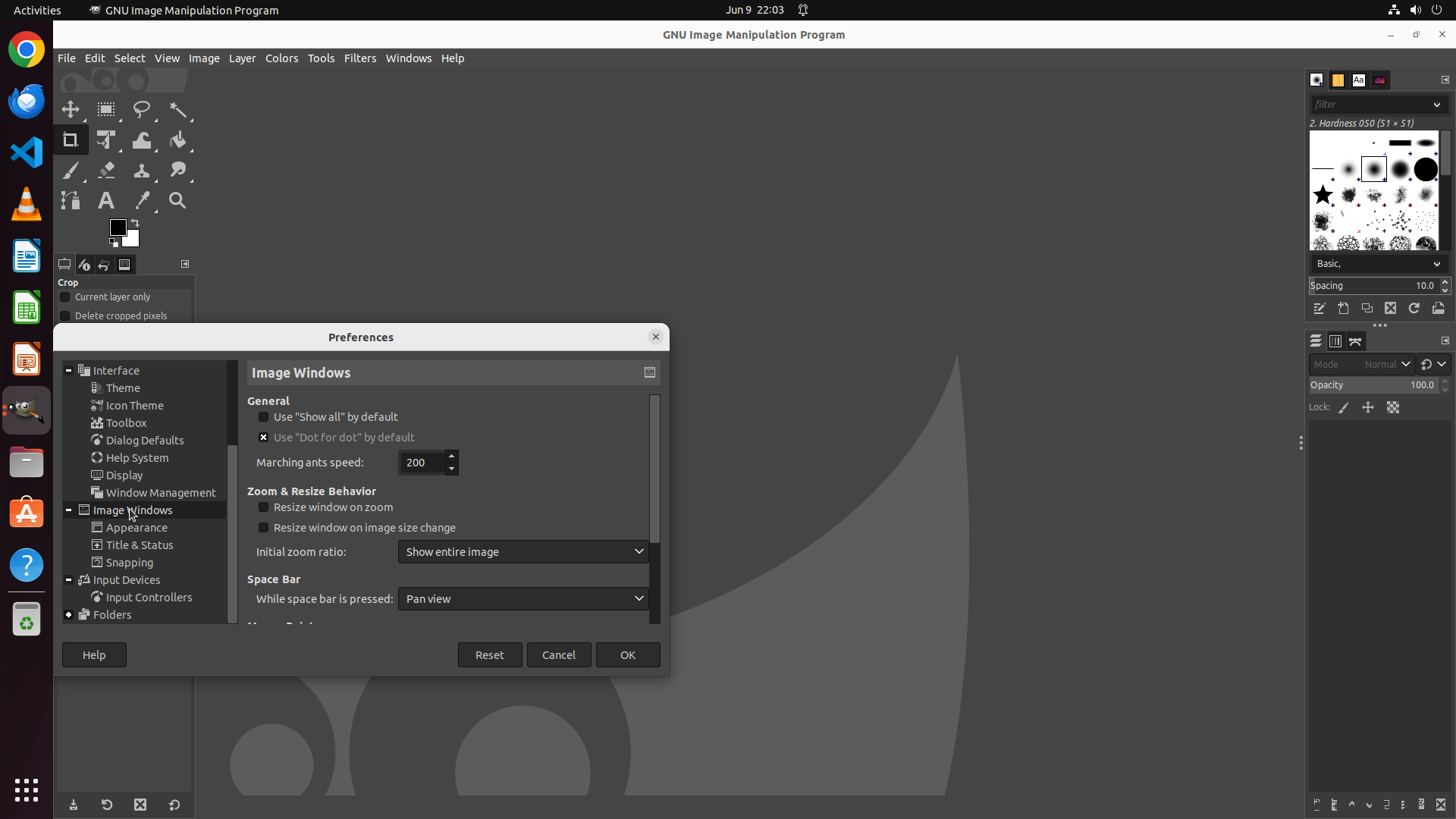}
\caption{An erroneous repeated action after the UI state changes (Step 17).}
\end{figure}

\paragraph{Negative Group: Different Failure Patterns Beyond Binary Rewards.}

While successful trajectories can differ in efficiency, failed trajectories can also exhibit substantially different qualities despite receiving the same binary reward. 
We illustrate this phenomenon using two failed trajectories generated by Qwen3-VL-8B-Thinking on the same \textit{Gimp} theme-switching task. 
Both trajectories obtain a reward of 0, but they exhibit different failure behaviors after deviating from the successful execution path.

For reference, the successful trajectory generated by LACL-GUI-8B changes the theme by selecting the ``Theme'' option at Step 4, followed by choosing the ``Light'' theme and confirming the change. 
Both failed trajectories diverge at this critical step: although the agent reaches the Preferences dialog and interacts with the Theme option, it fails to correctly complete the theme switch.

The first failure terminates prematurely after only four steps. 
The agent correctly navigates to the Theme option but incorrectly assumes that the task has been completed without verifying whether the desired state has been achieved.

\begin{tcolorbox}[
  colback=gray!8,
  colframe=gray!60,
  boxrule=0.4pt,
  arc=1pt,
  left=4pt,
  right=4pt,
  top=3pt,
  bottom=3pt,
  fontupper=\small\ttfamily
]
Step 1: Click the "Edit" menu in the top menu bar.\\
\quad pyautogui.click(124, 72)\\[2pt]

Step 2: Click the "Preferences" option in the Edit menu.\\
\quad pyautogui.click(180, 640)\\[2pt]

Step 3: Click the "Theme" option under the Interface section.\\
\quad pyautogui.click(157, 669)\\[2pt]

Step 4: Terminate (task incorrectly considered completed).\\
\quad DONE
\end{tcolorbox}

In contrast, the second failure continues interacting after the same divergence point. 
Although the agent performs additional actions, these steps do not recover the correct execution path and only introduce ineffective exploration before termination.

\begin{tcolorbox}[
  colback=gray!8,
  colframe=gray!60,
  boxrule=0.4pt,
  arc=1pt,
  left=4pt,
  right=4pt,
  top=3pt,
  bottom=3pt,
  fontupper=\small\ttfamily
]
Step 1-3: Navigate to the Theme option following the correct menu hierarchy.\\[2pt]

Step 4: Click the "Theme" option again in the Preferences dialog.\\
\quad pyautogui.click(167, 667)\\[2pt]

Step 5: Click the "OK" button without successfully applying the Light theme.\\
\quad pyautogui.click(830, 859)\\[2pt]

Step 6: Terminate (task incorrectly considered completed).\\
\quad DONE
\end{tcolorbox}

These examples reveal two common failure modes that are indistinguishable under binary rewards. 
The first corresponds to premature termination (result in early divergence from reference successful trajectory) after an incomplete state transition, while the second corresponds to unnecessary post-divergence interactions that increase trajectory length without improving task completion. 
Therefore, evaluating failed trajectories requires finer-grained preference modeling beyond success or failure labels. 
By incorporating trajectory structure and ineffective actions, LACL-GUI assigns more informative preferences within negative groups, encouraging the policy to avoid both early mistakes and unproductive exploration.

The screenshots below show representative failure cases. 
The first example demonstrates premature termination after the agent reaches the Theme option, while the second illustrates unnecessary post-divergence interactions.

\begin{figure}[H]
\centering
\includegraphics[width=0.92\columnwidth]{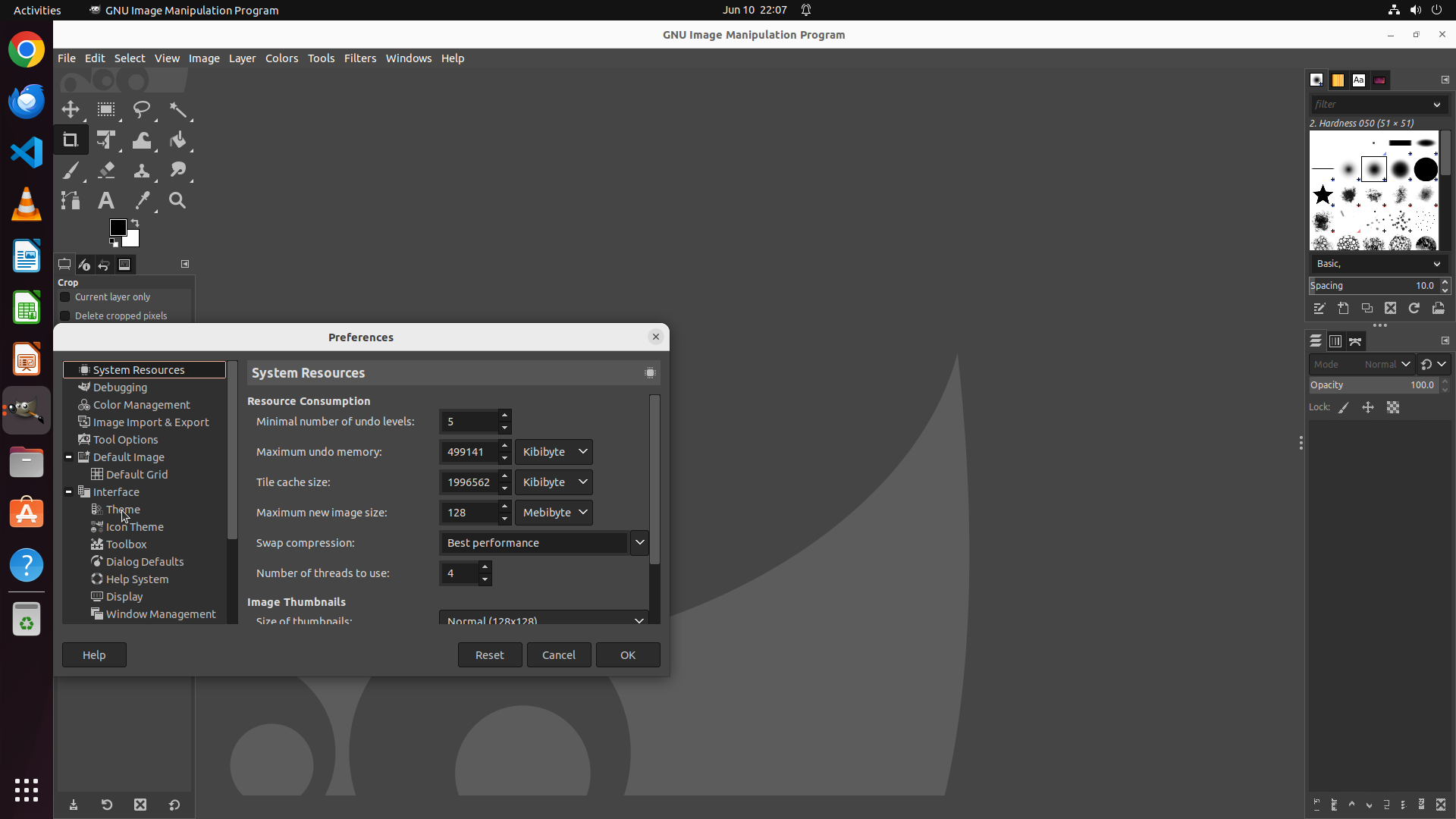}
\caption{Premature termination after an incomplete state transition (Step 4 in 4-step Failed Trajectory).}
\end{figure}

\begin{figure}[H]
\centering
\includegraphics[width=0.92\columnwidth]{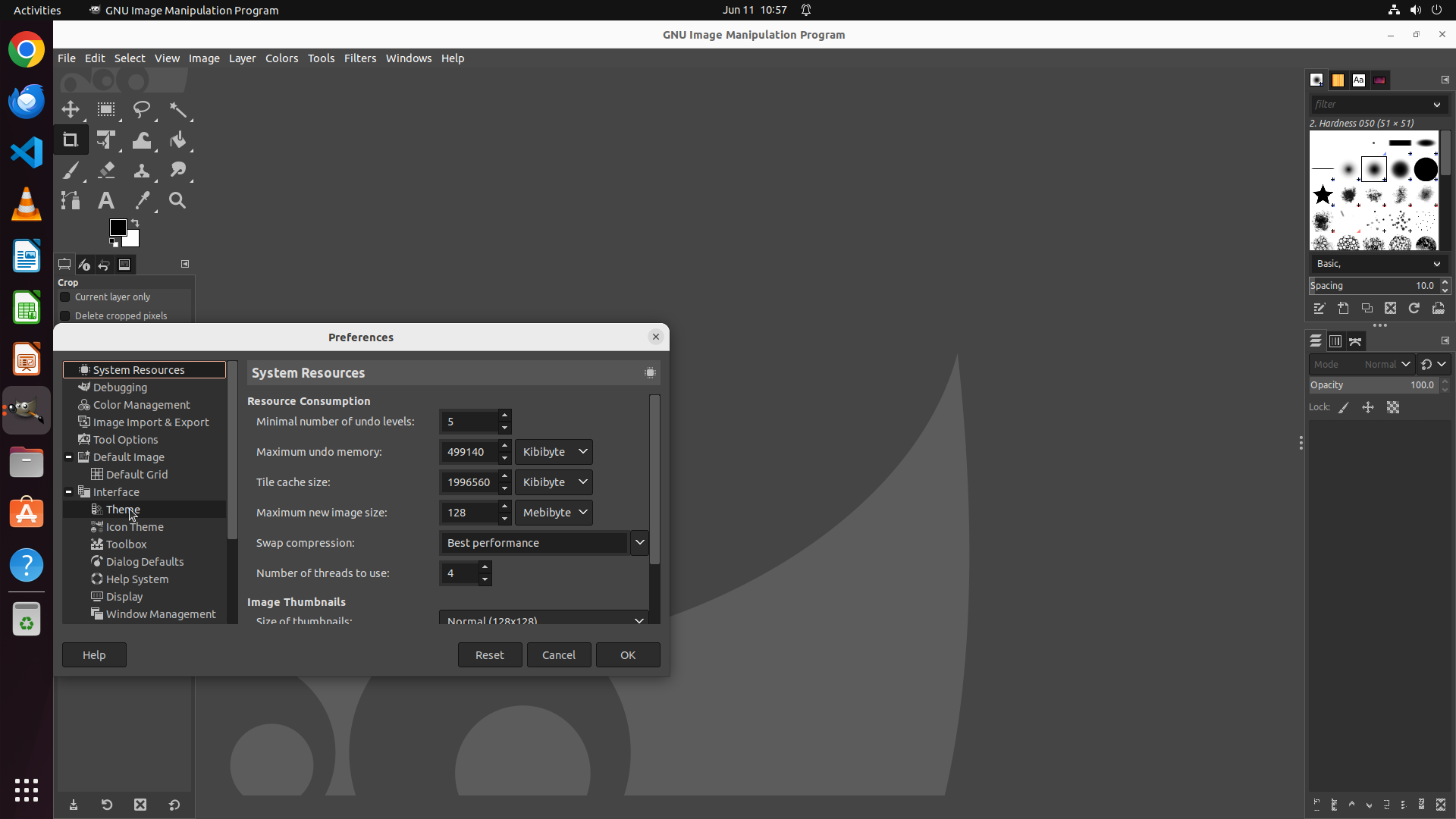}
\caption{Additional ineffective actions after divergence from the successful execution (Step 5 in 6-step Failed Trajectory).}
\end{figure}

These failure cases also demonstrate the motivation of the negative-group modulation in LACL-GUI. 
For each failed trajectory, LACL-GUI first identifies the divergence point from the reference successful trajectory using observation similarity and then estimates the number of erroneous steps after divergence. 
Although both trajectories deviate at the same decision point in this example, they exhibit different error severities: the 4-step trajectory terminates immediately after divergence, whereas the 6-step trajectory continues with additional ineffective interactions. 
Consequently, the latter receives a larger erroneous-step count and stronger suppression through the perturbed logits in Eq.~(9), preventing the policy from reinforcing long unsuccessful exploration. 
Meanwhile, this fine-grained modeling preserves informative failures that remain closer to successful execution, instead of treating all failed trajectories uniformly as in REAL.

\end{document}